\documentclass[a4paper,11pt]{article}
\usepackage{amsfonts, amssymb, graphicx, amsmath}
\usepackage[usenames]{color}
\usepackage[pdftex]{hyperref}
\usepackage{fullpage}
\usepackage{natbib}
\hypersetup{
    colorlinks=true,       
    linkcolor=blue,        
    citecolor=red,        
    filecolor=blue,        
    urlcolor=blue          
}

\usepackage{booktabs} 
\usepackage{array} 
\usepackage{paralist} 
\usepackage{verbatim} 
\usepackage{subfig} 
\usepackage{authblk}
\usepackage{fancyhdr} 
\newtheorem{theorem}{Theorem}

\newtheorem{lemma}{Lemma}

\newtheorem{remark}{Remark}

\newenvironment{proof}[1][Proof]{\textbf{#1.} }{\ \rule{0.5em}{0.5em}}

\newcommand{\argmin}{\mathop{\rm arg\min}}

\title{On Zeroth-Order Stochastic Convex Optimization via Random Walks}

\begin{document}
\date{}

\author[1]{Tengyuan Liang \thanks{tengyuan@wharton.upenn.edu}}
\author[2]{Hariharan Narayanan \thanks{harin@uw.edu}}
\author[1]{Alexander Rakhlin \thanks{rakhlin@wharton.upenn.edu}}
\affil[1]{Department of Statistics, The Wharton School, University of Pennsylvania}
\affil[2]{Department of Statistics and Department of Mathematics, University of Washington}

\maketitle

\begin{abstract}
	We propose a method for zeroth order stochastic convex optimization that attains the suboptimality rate of $\tilde{\mathcal{O}}(n^{7}T^{-1/2})$ after $T$ queries for a convex bounded function $f:{\mathbb R}^n\to{\mathbb R}$. The method is based on a random walk (the \emph{Ball Walk}) on the epigraph of the function. The randomized approach circumvents the problem of gradient estimation, and appears to be less sensitive to noisy function evaluations compared to noiseless zeroth order methods. 
\end{abstract}



\section{Introduction}

Let $f$ be a convex real-valued function on a closed convex domain $K\subset \mathbb{R}^n$. Within the oracle model of optimization, one sequentially obtains noisy information about this unknown function with the aim of computing an $\epsilon$-minimizer of $f$. In this paper we consider the setting of \emph{stochastic zeroth order optimization}: at step $t$, the oracle reveals a noisy value of the function at a point queried by the algorithm. This model is very basic, and can be viewed through the lens of learning: what is the amount of information one needs to collect in order to identify a near-optimal point, if all that is known about the objective is that it is convex?  

The amount of information can be quantified by the number of oracle calls, and it is known that this information-based complexity for minimization of a convex Lipschitz function $f$ to within an error $\epsilon>0$ scales as $\Omega(n^2\epsilon^{-2})$ \citep{Shamir12}. Yet, to the best of our knowledge, at present time there is no algorithm that comes even close to obtaining the desired $n^2$ dependence on the dimension while simultaneously having the $\epsilon^{-2}$ dependence on accuracy. The seminal work of \cite{NemYud83} introduces an optimization method for the noiseless zeroth order optimization with the $n^7$ dependence on the dimension, but the authors concede that extending this to the stochastic setting will worsen the power (this dependence is left as an unspecified polynomial in $n$). The noiseless zeroth-order method of \citep{NemYud83} was extended to the noisy case for the slightly harder problem of regret minimization in \citep{AgaFosHsuKakRak13siam}, where the authors proved an  $\tilde{\mathcal{O}}(n^{33}\epsilon^{-2})$ upper bound\footnote{The $\tilde{O}(\cdot)$ notation disregards polylogarithmic terms in $n$ and $1/\epsilon$.} on the \emph{regret} of the procedure, a more difficult objective. As a consequence, the same upper bound holds for the problem of optimization via averaging of the trajectory (see \citep{polyak1992acceleration, hazan2011beyond, AgaFosHsuKakRak13siam}). 

In the present paper, we describe a method with an $\tilde{\mathcal{O}}(n^{14}\epsilon^{-2})$ oracle complexity (or, equivalently, the $\tilde{\mathcal{O}}(n^{7}T^{-1/2})$ decay of suboptimality after $T$ steps). While not very practical for problems in high enough dimension, the method should be viewed as making progress towards closing the large theoretical gap. Further, the algorithm is based on random walks and is quite different from the classical techniques. These more classical approaches can be roughly divided into two categories: attempting to estimate the gradient using noisy function evaluations, or attempting to find a zero-th order method that is robust to noise. \cite{NemYud83} discuss the distinction between these two general plans of attack. The first is unlikely to yield the $1/\epsilon^2$ dependence on the accuracy, while the second appears to suffer from an adverse scaling with the dimension under noisy evaluations. The random walk approach can be viewed as yet another possible technique. While the present paper still leaves a large gap to the lower bound, there is hope for improvement using randomized methods such as a random walk. The reason we are optimistic about this approach is because randomized methods appear to be more robust to noise. Ideally, one would hope to use randomness in function evaluation as an asset rather than a disadvantage, thus ``riding on the noise''.

\begin{figure}[h]
	\centering
		\includegraphics[width=1.6in]{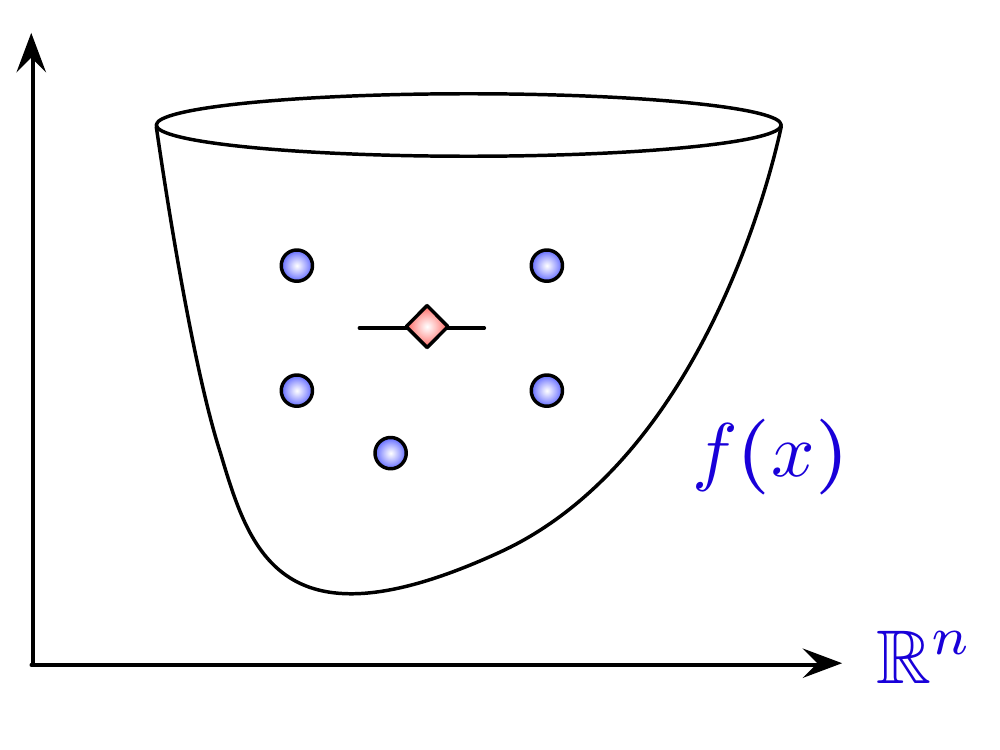}
	\caption{A nearly uniform sample is obtained via a random walk, and its average (red diamond) is computed. The convex set is then cut at the $y$-coordinate of the average point thus reducing the volume.}
	\label{fig:graphics_epigraph}
\end{figure}

Let us informally describe the method. We start with an $n+1$-dimensional convex body formed by the epigraph of the function and cut off at a value for the maximum of the function over the convex set $K$ (see Figure~\ref{fig:graphics_epigraph}). We run several Ball Walks on this body to generate near-uniform samples, in the spirit of the work of \cite{bertsimas2004solving}. Having obtained the samples, we take an average to approximate the center of mass, cut the convex set, and reshape. We continue in this fashion for $\tilde{\mathcal{O}}(n\log(1/\epsilon))$ iterations. The main technical difficulty is in analyzing the modified Ball Walk: unlike the related work of \cite{bertsimas2004solving} in the noiseless setting, we have a very restricted access to the convex body. To verify if the current point is inside the body, we adaptively sample the function value to obtain a confidence interval (in the ``vertical direction'') for the true value. If the confidence interval does not contain the current point, we proceed as if the inside/outside information were correct. Otherwise, we incur an additional error from not being able to resolve the question. Let us remark that our method is naturally parallelized, shaving off another factor of $n$ in terms of the number of queries per machine.


Let us mention very briefly that recent work has also considered a more restrictive oracle model, whereby one can obtain two function values with the same noise instance \citep{nesterov2011random,agarwal2010optimal,Duchi13}. This problem is markedly more simple, as one may form a good estimate of the gradient in any desired direction. Another simplifying assumption considered in the literature is an additional shape constraint such as smoothness of the objective (e.g. \cite{jamieson2012query}). In contrast, we only assume convexity and boundedness.


\section{Notation}
A noisy observation of the function value at a point $x\in K$ is denoted by $O_{1}\circ f (x)$, and the average of $m$ repeated queries is denoted by $O_m \circ f(x)$. We assume that the noise is sub-Gaussian with zero mean. We denote an optimal value $x^*\in\argmin_{x \in K} f(x)$. An affine transformation in $\mathbb{R}^{n+1}$ is denoted by $\mathcal{T} : \mathbb{R}^{n+1} \mapsto \mathbb{R}^{n+1}$. The region enclosed by the convex function $f$ and hyperplane $y\leq C_t$ is denoted by $K_t = \{(x,y)\in\mathbb{R}^{n+1}, C_t\geq y\geq f(x)\}$. It is easy to check that this is a convex set in $\mathbb{R}^{n+1}$. Let $\Delta \in \mathbb{R}^{n+1}$ denote the vertical vector linking $(x,f(x))\rightarrow (x,y)$ and $|\Delta|$ denotes its length, which equals to $y-f(x)$. We assume that an initial value $C_0$, an upper bound on the function over the set, is given. However, we do not assume that the function is Lipschitz.

%

Our goal is to bound the number of the noisy oracle calls given the target accuracy $\epsilon$. Assume that the convex set $K=K_0$ is contained in the axis-aligned cube of width $1$ centered at the origin. At the final epoch $T$, the remaining convex body $K_T$ will contain a cube of width $\epsilon$. 

 

\section{Random Walk on Convex Body}

In this section we will introduce the ball walk algorithm in the \emph{noiseless} oracle setting. Analysis of the Ball Walk algorithm was developed in \citet{kannan1997random}, and was later modified in \citet{bertsimas2004solving} to solve noiseless convex programs. Our algorithm for the noisy oracle setting builds on the theoretical properties of the noiseless Ball Walk.

The random walk algorithm consists of three main steps: ``Cut", ``Round" and ``Sample". We can assume that we start from $t=1$ and $C_0 = 1/2$, without loss of generality. Assume that $\tilde{\mathcal{O}}(n)$ near uniform distributed samples (a warm start that will be maintained throughout the procedure) are provided. This can be done in time that is independent of $\epsilon$. Let $n_t = \tilde{\mathcal{O}}(n)$ for all $t$. 
\begin{enumerate}
\item {\color{blue} Cut} the region $K_t$ at epoch $t$, enclosed by $y \geq f(\vec{x})$ and $y \leq C_t$, at the last coordinate of the average computed in the third step of the previous iteration. About $2/3$ of the random samples are still inside $K_t$, due to Lemma \ref{shrink.vol} below.
\item {\color{blue} Round} the convex body $K_t$ using an affine transformation $\mathcal{T}_t$ to a near isotropic position and denote the resulting convex body by $\mathcal{T}_t(K_t)$. The affine transformation $\mathcal{T}_t$ is calculated using half of the near-uniform samples left after the ``Cut'' procedure. ($1/3$ of the samples used in this step, and $1/3$ left untouched.)
\item {\color{blue} Sample} $n_{t} = \tilde{\mathcal{O}}(n)$ nearly independent uniform samples $X^t_1, X^t_2, ...., X^t_{n_t}$ (in the sense of Lemma \ref{mix.ball}) using ``Ball Walk" in the convex body $\mathcal{T}_t(K_t)$ based on the ``warm start" samples left after the ``Round'' procedure. Since there are about $\frac{1}{3} n_{t-1}$ of these seeds left, we run three independent chains of Ball Walk to ensure that we have $n_t$ samples after mixing. Set the new $C_{t} = \frac{1}{n_t} \sum_{i} (\mathcal{T}_t^{-1}\circ X^t_i )[n+1]$, here $X[n+1] \in \mathbb{R}$ denotes the last coordinate of the vector $X$. Go back to step 1.
\end{enumerate}

The following lemmas \ref{theta.lemma}-\ref{mix.ball} are useful in proving the theoretical guarantee of the Ball Walk algorithm.
The first lemma is taken from \citet[Corollary 5.2]{kannan1997random}.
\begin{lemma}[$\theta$-Near Isotropic]
\label{theta.lemma}
We call a convex body $K$ is in $\theta$-near isotropic position if for any vector $v$,
\begin{equation}
(1-\theta) \| v\|_2^2 \leq \mathbb{E} \langle x, v\rangle^2 \leq (1+\theta) \| v\|_2^2
\end{equation}  
where the expectation is taken over uniform distribution inside the convex body $K$.
Let $\theta<1/2$. If a convex body $K$ is in $\theta$-near isotropic position and $B$ is a unit ball, then
\begin{equation}
(1-2\theta)B \subseteq K \subseteq (1+2\theta) (n+1)B.
\end{equation} 
\end{lemma}
The second lemma assures the constant factor of volume shrinkage in each epoch. It is proved in \citet[Lemma 7]{bertsimas2004solving}.
\begin{lemma}[Shrinkage of Volume for Cutting]
\label{shrink.vol}
The volume of the covex body $K_t$ drops by a factor of $\frac{2}{3}$ with high probability in each epoch.
\end{lemma}
The third lemma is first introduced in \citet{kannan1997random} in a weaker version and later improved in \citet{rudelson1999random}. The version we are using can be found in \citet[Corollary 11]{bertsimas2004solving}.
\begin{lemma}[Near Isotropic Affine Transformation]
Let $K$ be a convex set, $$O(p n \log n \max\{p, \log n\})$$ random samples are sufficient to find an affine transformation $\mathcal{T}$ to bring $K$ into $1/4$-near-isotropic position, with probability at least $1-\frac{1}{2^{p-1}}$.
\end{lemma}
\begin{remark}
When we take $p= O(\log n)$, we can conclude that with overwhelming probability (with probability at least $1-\frac{1}{n^\alpha}$, where $\alpha$ is arbitrary in the sense that it only affects the constant in the big $O$ notation), the transformation $\mathcal{T}$ bring $K$ into near isotropic position with $O(n \log^3 n)$ random samples.
\end{remark}
The last lemma about the mixing time of Ball Walk is proved in \citet[Theorem 2.2]{kannan1997random}.
\begin{lemma}[Mixing Time for Ball Walk]
\label{mix.ball}
Given a convex body $K$ satisfying $B \subseteq K \subseteq d B$, a positive integer $N$ and $\epsilon>0$, we can generate a set of $N$ random points $\{v_1, \ldots, v_N\}$ in $K$ that are
\begin{itemize}
\item[(a)] almost uniform in the sense that the distribution of each one is at most $\epsilon$ away from the uniform in total variation distance, and
\item[(b)] almost (pairwise) independent in the sense that for every $1\leq i < j \leq N$ and every two measurable subsets $A$ and $B$ of $K$,
\begin{equation}
|P(v_i \in A, v_j \in B) - P(v_i \in A) P(v_j \in B)| \leq \epsilon.
\end{equation}
\end{itemize}
The algorithm uses only $\tilde{\mathcal{O}}(n^3d^2+N n^2 d^2)$ calls to the oracle.
\end{lemma}

Given the above lemmas, we can give a precise upper bound on the number of oracle calls need in the noiseless setting. The proof is given in \citet{bertsimas2004solving}. We sketch the main idea here for completeness.

\begin{theorem}[Bound on Oracle in Noiseless Case]
\label{noiseless.thm}
Each iteration of the random walk algorithm in the noiseless case uses at most $\tilde{\mathcal{O}}(n^4)$ number of oracle calls. Further, the algorithm can be implemented in at most $\tilde{\mathcal{O}}(n^5)$, with high probability.
\end{theorem}
\begin{proof}
According to \citet[Theorem 12]{bertsimas2004solving}, the number of oracle calls in each iteration is upper bounded by $\tilde{\mathcal{O}}(n^4)$. The volume ratio between the initial convex body and final convex body is
\begin{equation}
\frac{{\rm Vol} (K_T)}{{\rm Vol} (K_0)} > \epsilon^n.
\end{equation}
According to Lemma \ref{shrink.vol},
\begin{equation}
\left( \frac{2}{3} \right)^T \geq \frac{{\rm Vol} (K_T)}{{\rm Vol} (K_0)} > \epsilon^n
\end{equation}
Thus $T = O(n \log \frac{1}{\epsilon})$. So the total number of oracle calls is bounded by $\tilde{\mathcal{O}}(n^5)$.
\end{proof}

  
 \section{Adaptive Query Algorithm}
 
Based on the Ball Walk analysis in the noiseless setting, the remaining difficulty lies in bounding the misclassification error, i.e. the error of incorrectly classifying the query point as inside or outside the convex body. We use an adaptive query algorithm to address this problem near optimally. The intuitive statistical idea is: keep doubling the number of samples until we get enough ``confidence" to tell whether a point is inside or not. The adaptive query algorithm is illustrated in the following with full details.

Suppose we would like to decide whether a point $(x,y)\in \mathbb{R}^{n+1}$ is inside the current epigraph. The $m$-sample noisy oracle returns $O_m\circ f(x)$. Let $Z$ denote a standard normal $\mathcal{N}(0,1)$ (or sub-Gaussian tail random variable; the proof is almost identical), and $C$ a level to be determined later. We make the following adaptive decision:
\begin{itemize}
\item $(x,y)$ is {\color{blue} OUTSIDE} - {\color{blue} $Out$} if $y \leq O_m\circ f(x) - \frac{C}{\sqrt{m}}$. The probability of making this decision is 
$$P({\rm Outside}) = P\left(Z \geq \sqrt{m}(y-f(x))+C\right).$$
\item $(x,y)$ is {\color{blue} INSIDE} - {\color{blue} $In$}  if $y \geq O_m\circ f(x) + \frac{C}{\sqrt{m}}$. The probability of making this decision is 
$$P({\rm Inside}) = P\left(Z \leq \sqrt{m}(y-f(x))-C\right).$$
\item {\color{blue}POSTPONE} decision - {\color{blue} $Post$} decision if $y \in \left[O_m\circ f(x) - \frac{C}{\sqrt{m}},O_m\circ f(x) + \frac{C}{\sqrt{m}}\right]$. The probability of making this decision is $$P({\rm Postpone}) = P\left(\sqrt{m}(y-f(x))-C \leq Z \leq \sqrt{m}(y-f(x))+C\right).$$
\end{itemize}
The Adaptive Query algorithm is:
\begin{itemize}
\item[1.] Set a dictionary of sample size $m$ being the set $S = \left\{2^0, 2^1, \ldots, 2^{k},\ldots \right\}$.
\item[2.] Take $m$ from the dictionary and construct the test sequentially. Stop when we made a decision either {\color{blue}OUTSIDE} or {\color{blue}INSIDE}. Otherwise continuously increase $m$ from the dictionary.
\end{itemize}

\begin{lemma}
\label{adapt.query}
With probability at least $1 - 2\cdot \left( \log_2 \frac{4C^2}{|\Delta|^2}+1\right) \cdot \exp\left(-\frac{C^2}{2}\right)
  $, where $|\Delta| = y - f(x)$, $m = \frac{4C^2}{(y - f(x))^2}$ query numbers are enough to ensure correct decision.
\end{lemma}
\begin{proof}
First suppose the query point $(x,y)$ is inside the convex body, then $|\Delta| = y - f(x) >0$. Define the event $E = \left\{ \mbox{Classified as Inside using}~m \leq \frac{4C^2}{(y-f(x))^2}~\mbox{query samples} \right\}$  
\begin{align*}
P(E^c) & =  P(\text{Out}_{m=2^0}) + P(\text{Post}_{m\in S,m<2^1}, \text{Out}_{m=2^1})  \\
&~~~~~~~~~~ + \ldots+ P(\text{Post}_{m\in S, m<4C^2/|\Delta|^2}, \{\text{In}_{m = 4C^2/|\Delta|^2} \}^c) \\
&\leq P(\text{Out}_{m=2^0}) + P(\text{Out}_{m=2^1})+\ldots+P(\{\text{In}_{m = 4C^2/|\Delta|^2} \}^c)\\
& \leq  \sum_{i < \log_2 (4C^2/|\Delta|^2)} P\left(Z \geq \sqrt{m}(y-f(x))+C, m=2^i\right) \\
&~~~~~~~~~~ + P(Z \geq \sqrt{m}(y-f(x))-C, m=4C^2/|\Delta|^2)\\
& \leq \log_2 \frac{4C^2}{|\Delta|^2} \cdot \exp\left(-\frac{C^2}{2}\right) + \exp\left(-\frac{C^2}{2}\right) = \left( \log_2 \frac{4C^2}{|\Delta|^2} +1\right) \cdot \exp\left(-\frac{C^2}{2}\right).
\end{align*}
Similarly, suppose the query point $(x,y)$ is outside the convex body, then $|\Delta| = y - f(x) <0$. Define a set $E = \left\{ \mbox{Classified as Outside using}~m \leq \frac{4C^2}{(y-f(x))^2}~\mbox{query samples}\right\}$  
\begin{align*}
P(E^c) & =  P(\text{In}_{m=2^0}) + P(\text{Post}_{m\in S,m<2^1}, \text{In}_{m=2^1}) \\
&~~~~~~~~~~ +\ldots+ P(\text{Post}_{m\in S, m<4C^2/|\Delta|^2}, \{\text{Out}_{m = 4C^2/|\Delta|^2} \}^c) \\
&\leq  P(\text{In}_{m=2^0}) + P(\text{In}_{m=2^1})+\ldots+P(\{\text{Out}_{m = 4C^2/|\Delta|^2} \}^c)\\
& \leq  \sum_{i < \log_2 (4C^2/|\Delta|^2)} P\left(Z \leq \sqrt{m}(y-f(x))-C, m=2^i\right) \\
&~~~~~~~~~~ + P(Z \leq \sqrt{m}(y-f(x))+C, m=4C^2/|\Delta|^2)\\
& \leq  \log_2 \frac{4C^2}{|\Delta|^2} \cdot \exp\left(-\frac{C^2}{2}\right) + \exp\left(-\frac{C^2}{2}\right)=  \left( \log_2  \frac{4C^2}{|\Delta|^2} +1\right) \cdot \exp\left(-\frac{C^2}{2}\right).
\end{align*}
\end{proof}

\begin{remark}
As we can see, if $|\Delta|>\frac{1}{n^k}$ (polynomial decay in terms of $n$), as long as $C = O(\sqrt{\log n})$ with a constant big enough (say $C = \sqrt{2 (\ell+1) \log n}$), the error probability is $o(\frac{1}{n^\ell})$. This probability can be arbitrary small with polynomial decay in terms of of $n$.
\end{remark}


\section{Stochastic Convex Optimization}
The algorithm to solve stochastic convex optimization problem given noisy oracle in our paper is a combination of random walk in convex body and adaptive hypothesis testing, as illustrated in the following.

\begin{itemize}
\item[1.] Perform the random walk algorithm as in noiseless case.
\item[2.] When establishing whether a point is inside or outside, use the adaptive query algorithm.
\end{itemize}

In order to analyze the expected number of oracle calls used in this algorithm, we will first introduce some lemmas revealing the geometry of convex body and property of the level set function of the given convex function.
%
\begin{figure}[h]
	\centering
		\includegraphics[width=1.6in]{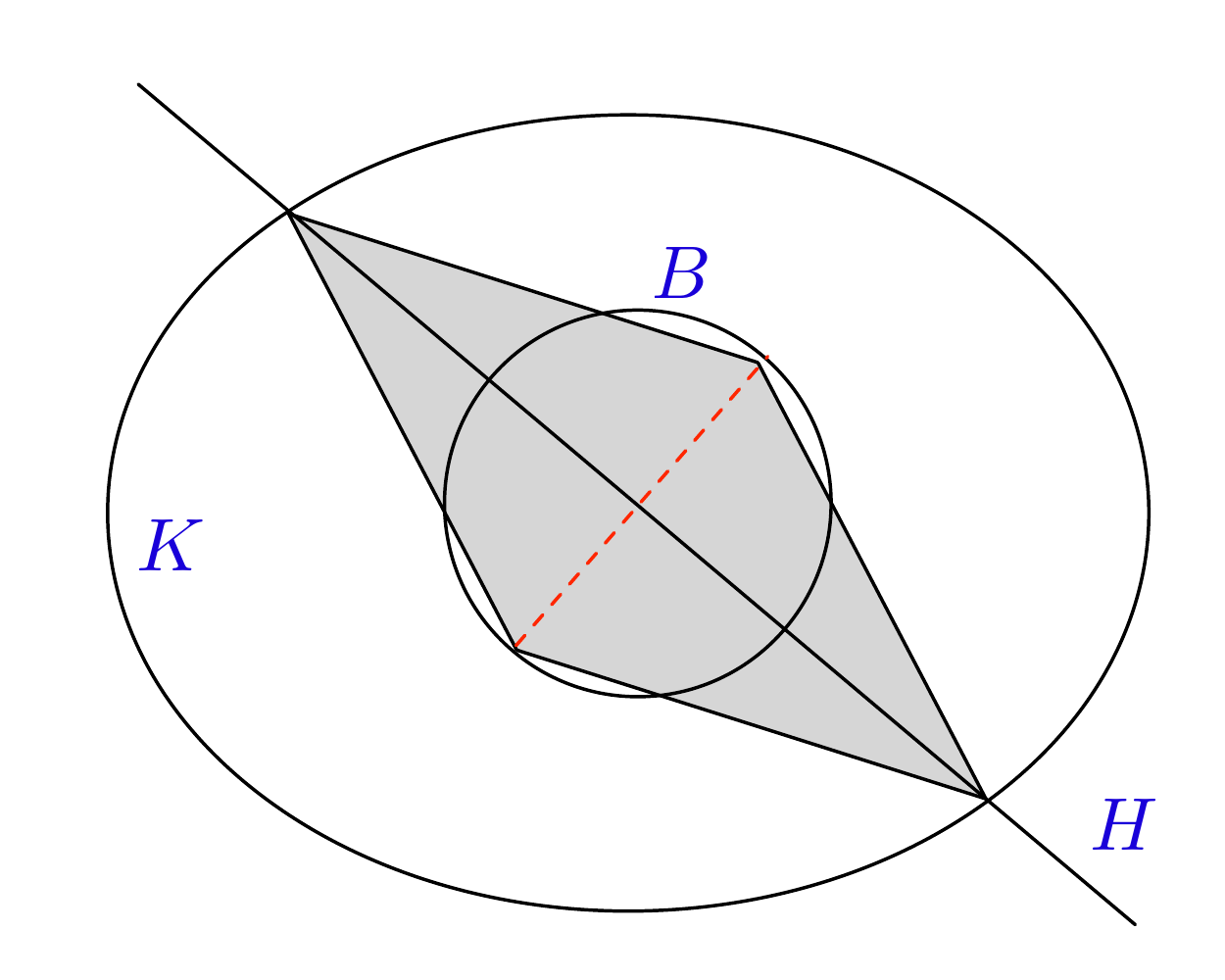}
	\caption{Graphical proof of Lemma \ref{geo.lm}. $K$ denotes the convex body, $B$ the inscribed ball, and $H$ the hyperplane. The shaded area is the ``Double Cone'' inside the convex body $K$.}
	\label{fig:graphics_cone}
\end{figure}

%

\begin{lemma}[Geometry of Convex Body]
\label{geo.lm}
Let $\theta<1/2$. For a convex body $K\in \mathbb{R}^n$ in $\theta$-near isotropic position,  and a arbitrary hyperplane $H \in \mathbb{R}^{n-1}$, the following inequality holds
\begin{equation}
{\rm Vol}(K \cap H) \leq \frac{n}{2-4\theta} \cdot {\rm Vol}(K)
\end{equation}
\end{lemma}
\begin{proof}
As shown in \citet{kannan1997random}, we can always find a unit ball $B$ inside the convex body $K$ such that
\begin{equation}
\label{iso.bound}
(1-2\theta)B \subseteq K \subseteq (1+2\theta) (n+1)B.
\end{equation}
Consider the intersection $K\cap H$. We know that this intersection is a convex body in $\mathbb{R}^{n-1}$. Find the farthest points $P$ (possibly two) on each side of $K \cap H$ on the surface of the unit ball $B$ to $K\cap H$. Connect $P$ with the boundary of $K\cap H$, we get a ``Cone" or ``Double Cone" $C$ with volume
\begin{equation}
\frac{2-4\theta}{n} \cdot {\rm Vol}(K \cap H) \leq \frac{(1-2\theta) \cdot d(B)}{n} \cdot {\rm Vol}(K \cap H) \leq {\rm Vol}(C) \leq {\rm Vol}(K).
\end{equation}
Thus proof completed. The graphical illustration of this lemma is in Figure~\ref{fig:graphics_cone}.
\end{proof}

\begin{lemma}[Distribution of the Convex Level-Set Function]
\label{dist.levelset}
Fix a convex function $f$ and the associated enclosed convex body $K$. For any point $(x,y) \in K$, $\Delta \in \mathbb{R}^{n+1}$ denotes the vertical vector linking $(x,f(x))\rightarrow (x,y)$. Denote the affine transformation that brings the convex body $K \subset \mathbb{R}^{n+1}$ to $\theta$-near isotropic position ($\theta<1/2$) as $\mathcal{T}$. Consider the uniform distribution on $\mathcal{T}(K)$. Then the distribution of the scalar $|\mathcal{T}(\Delta)|$ has the following properties:
\begin{itemize}
\item $|\mathcal{T}(\Delta)| = c \cdot |\Delta|$, with the constant factor depending on $\mathcal{T}$. 
\item The probability measure $\mathcal{P}(\cdot)$ of $|\mathcal{T}(\Delta)|$ satisfies $d \mathcal{P}(s) \leq \frac{n+1}{2-4\theta} \cdot ds$
\end{itemize}
\end{lemma}
\begin{proof}
Affine transformation keeps the ratio on a line and keeps the parallel property. So property 1 is proved. Property 2 is a direct consequence of Lemma \ref{geo.lm}, where the hyperplanes are the constructed according to the value $|\mathcal{T}(\Delta)|$.
\end{proof}

\begin{theorem}[Bound on Oracle in Noisy Case]
\label{noisy.thm}
There exist an event $E$ with probability at least $P(E) \geq 1 - o(1)$. On the event $E$, the expected number of noisy oracle calls is at most $\tilde{\mathcal{O}}(n^{14} \frac{1}{\epsilon^2})$.
\end{theorem}
\begin{proof}
	Following \citet[Theorem 4.1 Remark, Theorem 4.4]{kannan1997random}, samples from the random walk can achieve closeness to the uniform distribution in total variation sense very quickly. More explicitly, if we need $T$ steps to achieve a precision $\epsilon$, then $\epsilon/n^{10}$ can be achieved in almost the same number of steps: we lose only a factor of $O(\log n)$. Thus precision in total variation sense is not a crucial issue in sampling, and we can always assume the samples are drawn from the uniform distribution.

Next, a $1/4$-near isotropic position behaves like isotropic position in our complexity analysis: no additional $n$ factor is involved, and the only difference is in terms of the constant. Therefore, without loss of generality, we may assume the convex body is in the isotropic position after the transformation $\mathcal{T}$.

As we can see in Lemma \ref{adapt.query}, if the current query point is far away from the boundary along the vertical direction $\Delta$ (that is, $|\Delta|$ is large, which is equivalent to $|\mathcal{T}(\Delta)|$ being large), we can tell whether  or not the point is inside with high confidence within $\frac{4C^2}{|\Delta|^2}$ oracle calls. On the one hand, as the query point approaches the boundary, the number of oracle calls goes to infinity. One the other hand, the probability of getting very close to boundary is small. Thus in terms of theoretical analysis, there is a trade-off between whether we want to spend more oracle calls at a given point that is close to boundary or, alternatively, put the probability of the point being close to boundary into bad event scenario and thus save the oracle calls. Hence in analyzing the algorithm, there is a trade-off that determines the best point where we ``give up''. We will call this ``give up band'' as $\delta$-boundary in the following proof, as illustrated in Figure~\ref{fig:graphics_band}. We remark that after the $\mathcal{T}$ transformation, the direction in which we obtain noisy information is not vertical, but this does not impact the analysis.

\begin{figure}[h]
	\centering
		\includegraphics[width=1.6in]{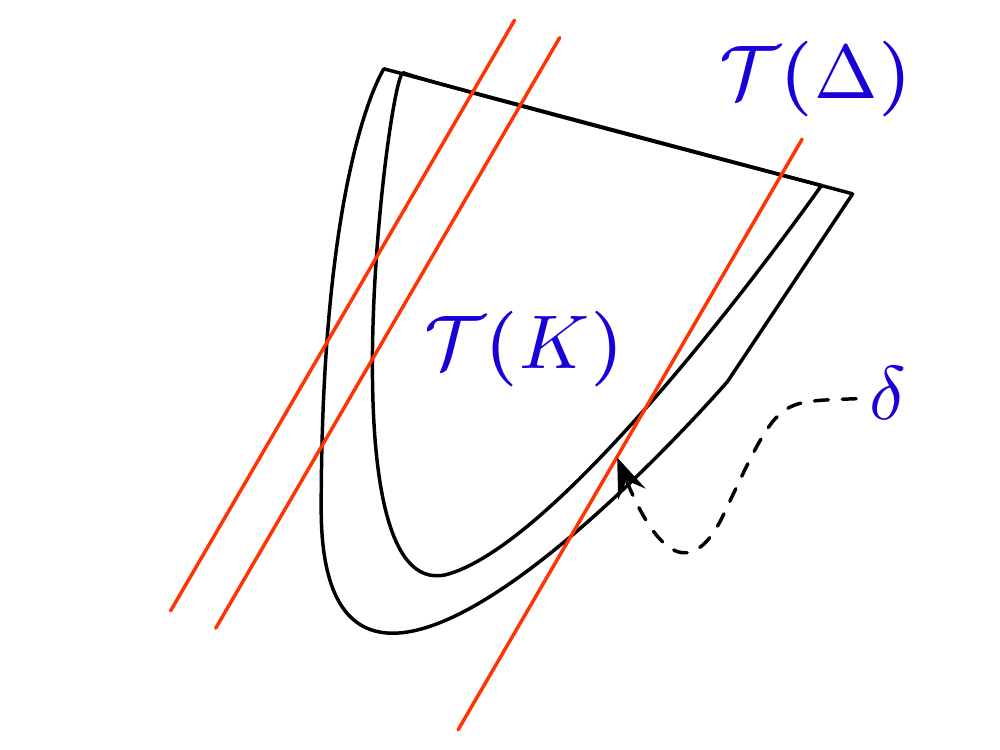}
	\caption{$\delta$-boundary illustration: the red lines denote the direction of the vector  $\mathcal{T}(\Delta)$, which is the direction in which we can do adaptive querying. Because of the affine transformation $\mathcal{T}$, it is not necessarily vertical. }
	\label{fig:graphics_band}
\end{figure}

According to Lemma \ref{dist.levelset}, we have $|\mathcal{T}(\Delta)| = c \cdot |\Delta|$, where $c$ is a constant factor. Because of \eqref{iso.bound}, and since an $\epsilon$ cube is inscribed at the final epoch, $c \leq \frac{n}{\epsilon}$. We now define a $\delta$-boundary as the mass $|\mathcal{T}(\Delta)| \leq \delta$. The procedure (for the purposes of analysis only) will give up on the $|\Delta| = y - f(x) = | \mathcal{T}(\Delta) | /c \leq \delta/c$ band and put this mass in our error term. (This giving up on the area near the boundary can be seen as the main source of looseness in the analysis, but we do not know how to avoid it).  
Then the error probability can be written in two parts, one coming from the probability of giving up in the band, the other coming from the statistical error of adaptive testing procedure. More precisely, 
\begin{align} 
P({\rm Error}) & \leq  P(|\mathcal{T}(\Delta)| \leq \delta) + P(E^c, |\mathcal{T}(\Delta)| \geq \delta)\\
& \leq \int_0^\delta d P(|\mathcal{T}(\Delta)|) + P(E^c, |\Delta| \geq \delta/c)\\
&  \leq n \delta + \left( \log_2 \frac{4C^2}{\delta^2 / c^2} +1\right) \cdot \exp\left(-\frac{C^2}{2}\right)
\end{align}
where event $E^c$ is defined in the same way as in Lemma \ref{adapt.query}.

Since the above error is per one step of the random walk, we need to take $\delta = \tilde{\mathcal{O}}(1/n^6)$ so that $n\delta$ accumulated after $\tilde{\mathcal{O}}(n^5)$ steps is $o(1)$. We have that $\frac{4C^2}{\delta^2 / c^2}$ is at most polynomial in terms of $n$. Hence, for $C = O(\sqrt{\log n})$ with constant big enough, we can ensure $P({\rm Error}) =  o(1/n^5)$. This last statement follows from Lemma \ref{dist.levelset}; we remark that this bound is sharp because when the convex body is a high dimensional cone, the bound is exact. This is the error for each point we query as in noiseless case. The total number of query in noiseless case is $\tilde{\mathcal{O}}(n^5)$, thus the total error behaves as $o(1)$ by choosing $\delta$ small and $C$ big.

On the complement of the ``Error" event, we have each step query complexity is bounded by $m = \frac{4C^2}{|\Delta|^2}$ (Lemma~\ref{adapt.query}). Thus the expected number of queries is $$\int_\delta^1 \frac{4C^2}{|\Delta|^2} d \mathcal{P}(|\mathcal{T}(\Delta)|).$$ Because of the uniform distribution on the convex body, the expected number of queries $\mathbb{E} N_q$ for each point is bounded by 
\begin{equation}
\mathbb{E} N_q  = \int_\delta^1 \frac{4C^2}{|\Delta|^2} d \mathcal{P}(|\mathcal{T}(\Delta)|) \leq \int_\delta^1 \frac{4C^2}{s^2 / c^2} \cdot  \frac{n+1}{2} d s \leq \tilde{\mathcal{O}}\left( n^9 \cdot \frac{1}{\epsilon^2}\right).
\end{equation}
(by Lemma \ref{dist.levelset} , the distribution of $|\mathcal{T}(\Delta)|$ has the relation $d \mathcal{P}(s) \leq \frac{(n+1)}{2} ds$.)
We conclude that the total number of oracle queries is $\tilde{\mathcal{O}}(n^5) \cdot \mathbb{E} N_q$ , which is $\tilde{\mathcal{O}}(n^{14} \frac{1}{\epsilon^2})$.
\end{proof}\\

\bibliographystyle{apalike}
\bibliography{Stochastic-Convex-Optimization}

\end{document}